\PassOptionsToPackage{table,prologue}{xcolor}
\documentclass[sigconf]{acmart}
\setcopyright{none}

\usepackage{booktabs} 
\usepackage{makecell}
\usepackage{amsmath}
\usepackage{multirow}
\usepackage{float}
\usepackage{enumitem}
\usepackage[ruled,vlined]{algorithm2e}
\usepackage{tabularx}
\newcolumntype{C}{>{\centering\arraybackslash}X}
\usepackage{graphicx}
\usepackage{subcaption} 
\usepackage{caption}

\newcommand{{\method}}{FairGuide}
\newtheorem{myDef}{Definition}
\newtheorem{myPro}{Problem}
\newtheorem{theorem}{Theorem}
\newtheorem{lemma}{Lemma}

\graphicspath{{figure/}{figures/}}

\begin{document}
\title{Let's Grow an Unbiased Community : Guiding the Fairness of Graphs via New Links}

\author{\textbf{Jiahua Lu}}
\affiliation{%
  \institution{Jilin University} 
  \country{}
}
\affiliation{%
  \institution{The Hong Kong University of Science and Technology (Guangzhou)} 
    \country{}
}
\email{lujh9920@mails.jlu.edu.cn}

\author{\textbf{Huaxiao Liu}}
\affiliation{%
  \institution{Jilin University}
  \country{}
}
\email{liuhuaxiao@jlu.edu.cn}

\author{\textbf{Shuotong Bai}}
\affiliation{%
  \institution{Jilin University} %
  \country{}
}
\email{baist20@mails.jlu.edu.cn}

\author{\textbf{Junjie Xu}}
\affiliation{%
  \institution{The Pennsylvania State University}
  \country{}
}
\email{junjiexu@psu.edu}

\author{\textbf{Renqiang Luo}}
\affiliation{%
  \institution{Jilin University} %
  \country{}
}
\email{lrenqiang@outlook.com}

\author{\textbf{Enyan Dai}}
\authornote{Corresponding author}
\affiliation{%
  \institution{The Hong Kong University of Science
and Technology (Guangzhou)} %
\country{}
}
\email{enyandai@hkust-gz.edu.cn}

\renewcommand{\shortauthors}{Lu et al.}

\begin{abstract}
Graph Neural Networks (GNNs) have achieved remarkable success across diverse applications. However, due to the biases in the graph structures, graph neural networks face significant challenges in fairness. Although the original user graph structure is generally biased, it is promising to guide these existing structures toward unbiased ones by introducing new links. The fairness guidance via new links could foster unbiased communities, thereby enhancing fairness in downstream applications. To address this issue, we propose a novel framework named {\method}. Specifically, to ensure fairness in downstream tasks trained on fairness-guided graphs, we introduce a differentiable community detection task as a pseudo downstream task. Our theoretical analysis further demonstrates that optimizing fairness within this pseudo task effectively enhances structural fairness, promoting fairness generalization across diverse downstream applications. Moreover, {\method} employs an effective strategy which leverages meta-gradients derived from the fairness-guidance objective to identify new links that significantly enhance structural fairness. Extensive experimental results demonstrate the effectiveness and generalizability of our proposed method across a variety of graph-based fairness tasks.
The codes and datasets are available at \url{https://github.com/ljhds/FairGuide}. 

\end{abstract}



\maketitle

\section{INTRODUCTION}

Graph data has become an essential part of many applications like social networks analysis~\cite{kumar2022influence}, recommendation systems~\cite{wu2022graph},  and  fraud detection~\cite{cheng2020graph}. In graph data, nodes typically represent entities or individuals, while edges capture the relationships between them. Graph Neural Networks (GNNs) have emerged as powerful tools for leveraging both node features and graph topology to perform tasks such as node classification~\cite{kipf2016semi,rong2020dropedge}, graph embedding~\cite{ying2018hierarchical,Zhu:2020vf}, and link prediction~\cite{zhang2018link}, leading to significant improvements in task performance. 

Despite the great performance of GNNs, linking biases in user graphs raise fairness concerns when deploying GNNs in critical applications. Specifically, as illustrated in Fig.~\ref{fig:intro}, online social networks and residential communities often exhibit pronounced structural barriers, with tightly interconnected subgroups interacting primarily within themselves, resulting in inequitable access to resources, opportunities, and influence~\cite{saxena2024fairsna}. Such biases originating from data can be further amplified by the message-passing mechanisms of GNNs. For instance, recommendation systems on social networks have been found to systematically prevent female profiles from becoming among the most commented or liked~\cite{bose2019compositional}. Similarly, GNN-based book recommendation systems can be biased toward recommending books authored by males~\cite{buyl2020debayes}. 

Biases inherent in user connections raise a critical yet underexplored question: \textit{Can we guide existing real-world user graphs toward unbiased structures, thereby ensuring fairness in the deployment of GNN-based applications on these graphs?}  
To address question, a  promising strategy is to introduce new connections between users to guide the growth of graphs toward unbiased communities.
As shown in Fig.~\ref{fig:intro}, recommending connections between users from distinct groups can effectively break structural barriers. Such fairness guidance via new links could foster unbiased communities in real-world scenarios. Consequently, GNN classifiers trained on the unbiased structures can yield fair outcomes for various downstream tasks.  Therefore, it is crucial to investigate the problem of guiding graphs toward fairness by introducing new links.

\begin{figure}[t]
    \centering
    \includegraphics[width=\linewidth]{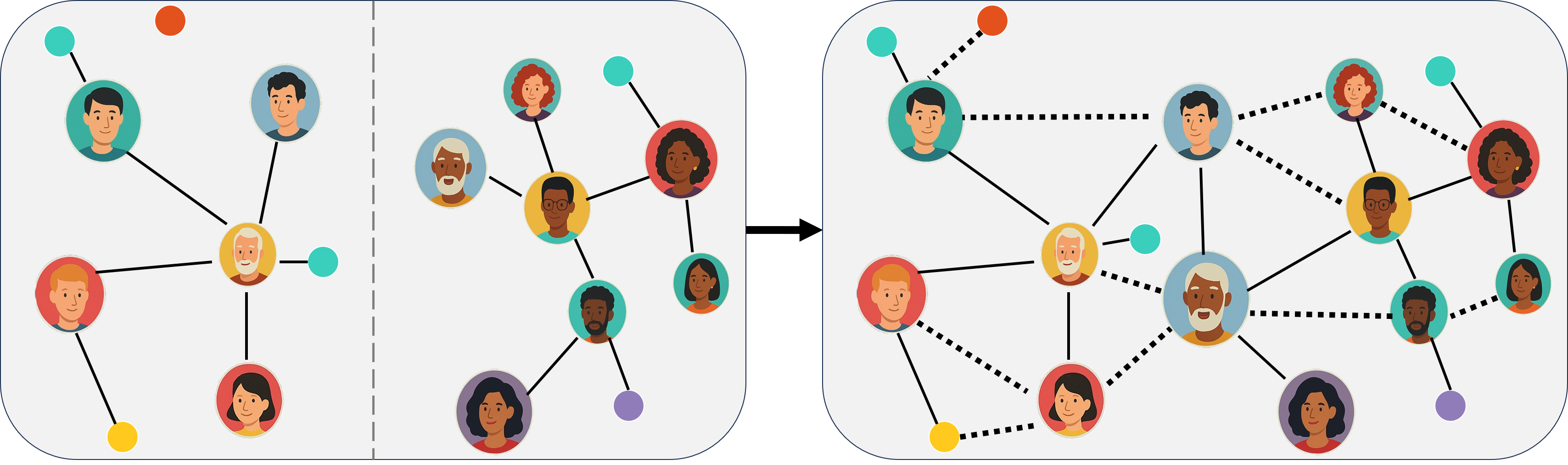} 
    \begin{subfigure}[b]{0.22\textwidth}
        \caption{Original Graph}
        \label{fig:biased}
    \end{subfigure}
    \hfill
    \begin{subfigure}[b]{0.22\textwidth}
        \caption{Fairness-Guided Graph}
        \label{fig:unbiased}
    \end{subfigure}
    \vspace{-1em}
    \caption{Illustration of guiding the fairness of graphs via new links. (a) An example of a biased social network with structural barriers between user subgroups; (b) Introduction of new links among subgroups to foster unbiased community.}
    \vspace{-1em}
  \label{fig:intro}
\end{figure}

Extensive research has been conducted to address fairness issues in graph neural networks. Specifically, in-processing ~\cite{buyl2020debayes,wang2022unbiased,zhao2022towards,zhu2024one,zhu2024devil,ling2023learning,luo2024fugnn} and post-processing~\cite{kang2020inform,masrour2020bursting,bose2019compositional} approaches primarily modify the training procedures or refine the output predictions of GNNs to mitigate algorithmic biases.  Additionally, pre-processing~\cite{dong2022edits,spinelli2021fairdrop,li2021dyadic} typically provide model-agnostic solutions by  modifying  the graph structures or node features to achieve fairness.
Despite these significant advances, they are generally not applicable to the problem of guiding user graphs toward fairness via new links.  \textit{Firstly}, in-processing and post-processing approaches primarily focus on obtaining GNN classifiers that are fair with respect to the given task ~\cite{dong2023fairness,zhu2024devil,wang2022unbiased}. Updating the graph students for fairness are not considered in these approaches.
\textit{Secondly}, though pre-processing approaches for fair GNNs~\cite{dong2022edits,li2021dyadic} also involve modifying the input graphs, these approaches typically impose no explicit constraints on structural modifications. As a result, these methods often result in link removals and excessive link suggestions for individual users. This hinder the application of pre-processing to guiding the real-world growth of user graphs toward fairness, as users are generally reluctant to remove existing connections or accept numerous new recommendations. Furthermore, these methods typically optimize structures for specific downstream tasks, making them unsuitable for fairness guidance without predefined tasks.

Therefore, we aim to design a framework to guide the fairness of user graphs via new links. However, this goal is non-trivial, where two main challenges remain to be addressed. (i) The fairness guidance aims to reduce the inherent structural biases by selectively suggesting new links. Therefore, guidance on graph growth toward fairness should be task-agnostic. It remains challenging to effectively guide user graphs toward fairness without knowledge of specific downstream tasks.
(ii) As discussed, users are unlikely to accept a large number of link suggestions. Thus, another crucial challenge is how to effectively guide graphs toward fairness by introducing only a limited number of new links.
To address these challenges, we propose a novel framework named {\method}. Specifically, to eliminate structural biases without specific downstream tasks, {\method} leverages community detection as a pseudo downstream task. Since labels in many downstream tasks are closely correlated with community structures, ensuring structural fairness for community detection intuitively benefits various downstream applications. This intuition is further supported by our theoretical analysis.
To effectively utilize the limited link-addition budget, {\method} employs a module to identify optimal new links for structural fairness. Specifically, an efficient meta-gradient computation method is deployed to approximate the impact of potential link additions on structural fairness. This enables recommending optimal links to guide the graph toward fairness.
In summary, our main contributions are:
\begin{itemize}[leftmargin=*, topsep=10pt]
    \item We focus on a novel problem of fairness guidance, which aims to guide the existing biased graph structures into fair community via introducing new links.
    \item We propose a novel fairness guidance framework named {\method}, which incorporates a pseudo downstream task and link addition through meta gradients to identify optimal new links for structural fairness.
    \item We collect a new large-scale social network from GitHub to provide empirical validation of {\method} in a real-world scenario.
    \item Both theoretical analysis and empirical results demonstrate the our {\method} can effectively guide graphs toward fairness to facilitate the fairness of downstream tasks.
\end{itemize}

\section{RELATED WORK}
In this section, we introduce the fair graph learning and link prediction methods that are closely related to our work.  
\subsection{Fair Graph learning}
Graph learning models have been widely adopted for analyzing topological data, demonstrating outstanding performance in various graph-based tasks \cite{kipf2016semi,zhang2018link,tsitsulin2023graph}. However, recent studies~\cite{kang2020inform,ma2022learning} reveal that fairness concerns emerge prominently in graph learning models. For instance, FairGNN~\cite{dai2021say} demonstrates that biases can implicitly propagate through graph structures, while  EDITS~\cite{dong2022edits} further demonstrates that biased graph topologies directly lead to discriminatory model outcomes. To address such challenges, numerous approaches have been proposed to improve fairness in graph learning. In-processing methods typically employ techniques like adversarial training~\cite{ling2023learning,11037524}, fairness-aware regularization~\cite{bose2019compositional,jiang2024chasing}, invariant learning~\cite{zhu2024one}  to reduce bias of specific models.  Some works also try to learn a fair graph representation by forcing distribution alignment~\cite{li2024graph} or preventing sensitive information leaking~\cite{zhu2024devil}. Additionally, pre-processing strategies modify the graph data itself, such as edge rewiring~\cite{spinelli2021fairdrop,li2021dyadic,wang2025fair} and feature masking~\cite{ling2023learning}, or whole data reconstruction~\cite{dong2022edits} to reduce the data distribution gap of different group and create fair input graphs for downstream tasks.And some data-oriented methods achieve fairness at a lower cost through data rebalancing~\cite{li2024rethinking}.

However, existing approaches face inherent limitations in solving the  problem of guiding the fairness of graphs. In-processing methods are limited to debiasing graph neural networks, making them inapplicable for addressing biases in the graph structures. Pre-processing methods aim to achieve fairness by modifying graph data directly but often involve impractical link removal and lack explicit constraints on the amount of structural modifications, which is not suitable for the real-world community. 
By contrast, our method explores to obtain an unbiased graph network by guiding the original graph data towards a fairer state, which remains under-explored for prior works.

\subsection{Link Prediction}
Standing as one of the most fundamental tasks in graph representation learning, Link prediction has been widely used in social network to help users find people that have not been connected~\cite{su2020link,daud2020applications}.  Existing approaches can be broadly categorized by their underlying techniques, including similarity-based methods~\cite{yu2017similarity,rossi2021knowledge}, dimensionality reduction-based methods~\cite{du2020cross} and so on. Among these, similarity-based methods are the most prevalent and can be further divided into two paradigms: structure-based methods ~\cite{aziz2020link,luo2021link} and attribute-aware methods ~\cite{xiao2021link,zhang2023iea}. Structure-based methods  rely on topological similarity metrics to infer potential links while Attribute-aware methods incorporate node features or semantic attributes to enhance prediction accuracy. However, such similarity-based approaches face inherent fairness-related limitations. Due to the inherent similarity in the same group, traditional link prediction techniques are difficult to break the bubbles in the community~\cite{masrour2020bursting} and may even enlarge the bias. Some works have investigated fairness-aware link prediction or recommendation systems to mitigate algorithmic bias and improve diversity in recommended connections~\cite{li2021dyadic,li2022fairlp}. However, these works primarily focus on addressing specific link prediction or recommendation tasks, rather than providing systematic frameworks to help guide a fair and diversity graph network. 
\section{PRELIMINARY ANALYSIS}

In this section, we present preliminaries of the fairness issues in the real-world user graph structures. 

\subsection{Notations} 
\label{sec:notations}

We use $\mathcal{G}=(\mathcal{V}, \mathcal{E}, \mathbf{X})$ to denote user graph where $\mathcal{V}=\{v_1,...,v_N\}$ is the set of $N$ nodes, representing the users in the network, $\mathcal{E} \in \mathcal{V} \times \mathcal{V}$ is the set of edges, and $\mathrm{\mathbf{X}} \in \mathbb{R}^{N \times M}$ is the node attribute matrix, and each node $v_i$ is associated with a $M$-dimensional feature vector ${\mathbf{x}_i}$.  $\mathrm{\mathbf{A}} \in \mathbb{R}^{N \times N}$ is the adjacency matrix of $\mathcal{G}$, where $\mathbf{A}_{ij}=1$ if nodes $v_i$ and $v_j$ are connected, otherwise $\mathbf{A}_{ij}=0$.  ${Y}$ and $\hat{{Y}}$ represent the ground truth and outcomes for the downstream task, respectively. In this work, we focus on binary sensitive attribute which is denoted as $s \in \{0,1\}$.

\subsection{Collection of Github Dataset}
For the purpose of this study, we constructed a real-world dataset by crawling from the social platform \textbf{Github}, which is the most popular open source platform in the world. This Github dataset contains more than 30,000 developers' profiles and 270,000 links between users. The features of developers include gender, region , followers, development language and etc. And the links stands for the following relationship between two developers.  Both features and links of the dataset is collected by Github REST API. We divide  users into \textit{developed} and \textit{developing} according to the countries they belong to and treat it as sensitive attribute. Then we further categorize a user with more than 35 followers as popular developer and treat it as the predicting target for the classification task.

\subsection{Preliminaries of Fairness}
\label{sec:fair_definition}
We focus on two widely used group fairness notions, i.e., statistical parity~\cite{dwork2012fairness} and equal opportunity~\cite{hardt2016equality}. And we consider a binary sensitive attribute, i.e., $s \in \{0,1\}$.

\begin{myDef}[Statistical Parity]
Statistical parity ensures that the predicted result $\hat{Y}$ is independent of the sensitive attribute $s$. Formally, for a binary classification task, i.e, $\hat{Y} \in \{0, 1\}$, the metric of statistical parity is computed by:
\begin{equation}
\begin{aligned}
 \Delta_{SP} = |P ( \hat{Y}=1| s=0 )-P ( \hat{Y}=1| s=1 )|.
\end{aligned} 
\label{eq5}
\end{equation}
For a multi-category task, i.e., $\hat{Y} \in \{0,1,...,C\}$ , $\Delta_{SP}$ extends to:
\begin{equation}
\begin{aligned}
\Delta_{SP} = \frac{1}{2} \sum_{C_k \in \mathcal{C}} \Big| P(\hat{Y} = C_k \mid s = 0) - P(\hat{Y} = C_k \mid s = 1) \Big|
\end{aligned} 
\label{eq5}
\end{equation}
A lower value of $\Delta_{SP}$ indicates fairer predictions with respect to the sensitive attribute.
\end{myDef}

\begin{myDef}[Equal Opportunity]  
Equal opportunity requires that the predicted result $\hat{Y}$ is conditionally independent of the sensitive attribute $s$ , given the true label ${Y}$. Considering a binary classification task, the metric of equal opportunity is formulated as:
\begin{equation}
\begin{aligned}
 \Delta_{EO} = | P ( \hat{Y}| s=0, Y=1 ) - P ( \hat{Y}| s=1, Y=1 ) |
\end{aligned} 
\label{eq5}
\end{equation}
Similar to $\Delta_{SP}$, lower $\Delta_{EO}$ implies fairer results of predictions. 
\end{myDef}

\subsection{Discrimination Analysis on Github Dataset}
\label{sec：dis_analysis}

To analyze the biases of the user graph structure, we adopt GCN~\cite{kipf2016semi} to perform node classification task and Louvain algorithm~\cite{blondel2008fast} to perform community detection task.
From Tab.~\ref{tab:preliminary}, we can observe that both $\Delta_{SP}$ and $\Delta_{EO}$ exhibit high values in the node classification and community detection tasks, indicating significant biases within the original graph structure.

Since we aim to guide the fairness of graphs via new links, we conduct preliminary analysis to investigate how simple link addition strategies will affect the performance and fairness on downstream tasks. Specifically, two link addition strategies are considered. (i) \textbf{Link Pred.}: It adds edges by a link predictor~\cite{schlichtkrull2018modeling} ; (ii) \textbf{Rand. Add}: It randomly adds edges to the Github graph;
Both strategies are constrained to add 4\% number of existing links. 
The results on node classification and community detection are presented in Tab.~\ref{tab:preliminary}, where we can observe that: 
\begin{itemize}[leftmargin=*]
    \item Adding edges via link prediction improves task performance but has limited effectiveness in mitigating bias. This is because connecting similar nodes primarily reinforces existing community structures, thus preserving underlying structural biases; 
    \item Adding some random links is slightly effective in mitigating bias. This is because these random edges can lead to inter-group connections.  However, the biases are still significant and the task performance degrade largely.
\end{itemize}
These observations indicate that existing link prediction methods are insufficient to address the fairness guidance problem.

\begin{table}[t]
\small
\centering
\caption{Discrimination of GCN in tasks of node classification and community detection (CD) on the Github dataset.}
\vspace{-1em}
\begin{tabular}{lrcccc} 
\toprule
Tasks & Metrics & Original & Link Pred. & Rand. Add \\ 
\midrule 
\multirow{4}{*}{\makecell[l]{Node \\ Classification}} & F1 (\%)  $\uparrow$ & 78.6 $\pm$ 0.2 & 78.8 $\pm$ 0.1  &    78.0 $\pm$ 0.1          \\
                         & AUC (\%) $\uparrow$ & 85.4 $\pm$ 0.5 & 85.4 $\pm$ 0.1  & 84.3 $\pm$ 0.1 \\
                          & $\Delta_{SP} (\%) \downarrow$  & 12.5 $\pm$ 0.4 & 11.9 $\pm$ 0.2 & 11.0 $\pm$ 0.2 \\
                          & $\Delta_{EO} (\%) \downarrow$  & 8.5 $\pm$ 0.5 & 8.3 $\pm$ 0.4 & 8.3 $\pm$ 0.3  \\
\midrule
CD
                          & $\Delta_{SP} (\%)\downarrow$  & 41.9 $\pm$ 1.5 & 38.2 $\pm$ 4.0 & 35.7 $\pm$ 2.8  \\

\bottomrule
\end{tabular}
\label{tab:preliminary}
\end{table}

\subsection{Problem Definition}

The analysis presented in Sec.~\ref{sec：dis_analysis} demonstrates the necessity of developing methods that guide graph structures toward fairness by introducing new links. With the notations in Sec.~\ref{sec:notations} and fairness definitions in Sec.~\ref{sec:fair_definition}, the fairness guidance via link addition can formulated as:
\begin{myPro}[Fairness Guidance via Link Addition] 
Given a user graph $\mathcal{G} = (\mathcal{V}, \mathcal{E}, \mathbf{X})$ with  sensitive attributes $\mathcal{S}$, our objective is to strategically add $n$ links to obtain a fair graph ${\mathcal{G}}' = (\mathcal{V}, \mathcal{E}', \mathbf{X})$.The number of added links, denoted as $n$, is subject to a predefined constraint  $\Delta$. A GNN model $f: {\mathcal{G}}' \rightarrow \hat{y}$, trained on the modified graph ${\mathcal{G}'}$, is required to produce predictions $\hat{y}$ that satisfy the fairness criteria introduced in Sec.~\ref{sec:fair_definition}.
\end{myPro}
\begin{figure}[t]  
  \centering
  \includegraphics[width=\linewidth]{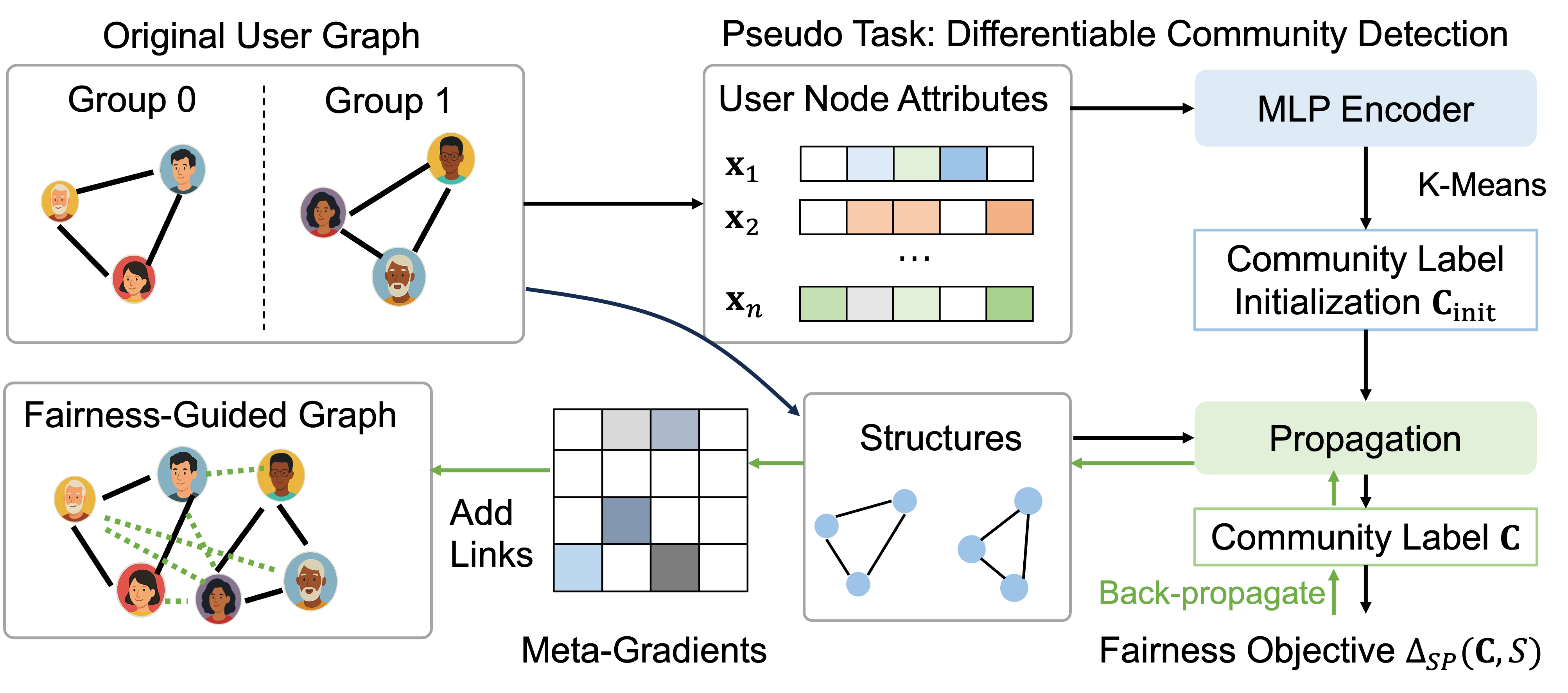} 
  \caption{An illustration of {\method} which add new links to guide the growing of graph structures toward fairness.}
  \label{fig:pipeline}
\end{figure}

\section{METHODOLOGY}

In this section, we present the details of {\method}. {\method} formulates fairness guidance via new links as an optimization problem, which selects new links that  minimize the structural biases. 
Two major challenges remain to be solved: (i) how to measure the biases of structures without accessing to the downstream tasks; (ii) how to add links to effectively guide the fairness of the user graph structures within a constrained budget for link addition. 
As illustrated in Fig.~\ref{fig:pipeline}, {\method} leverages community detection as a pseudo task to reflect the structural biases for downstream tasks. 
This pseudo downstream task serves as a proxy for downstream tasks, allowing us to generalize fairness optimization without relying on task-specific adjustments. 
To fully utilize the limited budget for link addition, {\method} employs meta-gradients derived from the fairness-guidance objective to identify new links that most effectively enhance structural fairness. Next, we introduce the bi-level optimization objective function of {\method} followed by details of each component.

\subsection{Overall Goal of {\method}}

The goal of {\method} is to add fairness-guiding links to $\mathcal{G}$ to  ensure the GNN classifier trained on the updated graph structure $\mathcal{G}'$ produces fair predictions.
Let $\mathcal{Y}_t$ and $\mathcal{L}_t$ denote the labels and the training loss for the target downstream task, respectively.
Fairness guidance via the introduction of new links can be formulated as the following bi-level optimization on the structures:  
\begin{equation}
\begin{aligned}
    \min_{{\mathcal{G}'}} &  \mathcal{M}(f_{\theta^*}(\mathcal{G}'),  \mathcal{S}, \mathcal{Y}_t) \\
    \text{s.t.}  \quad & \theta^*  = \arg\min_\theta \mathcal{L}_{{t}}(f_\theta(\mathcal{G}'), \mathcal{Y}_t) \\
    & {\mathcal{E}'} \supseteq \mathcal{E}~~\text{and} ~~ {\mathcal{E}'}-\mathcal{E} \leq \Delta 
\end{aligned}
\label{eq:overall}
\end{equation}
where $\mathcal{M}$ measures the fairness of $f_{\theta^*}$ on the updated graph structures $\mathcal{G}'$ that have been added links. $\mathcal{L}_t$ denotes the training loss of the downstream GNN classifier. The constraint $\mathcal{E}'-\mathcal{E} \leq \Delta$ limits the number of added links to at most $\Delta$. In the inner-level optimization, it simulates the training of the GNN classifier on the user graph for the target downstream task. In the outer-level optimization,  the optimal set of new links are identified to promote fairness in the downstream GNN classifier.

\subsection{Pseudo Downstream Task}

With the Eq.(\ref{eq:overall}), the fairness guidance via link addition is formatted to a bi-level optimization on the graph structures under constraints. However, the Eq.(\ref{eq:overall}) requires to measure the fairness based on GNN predictions on downstream tasks, which are unavailable during the fairness guidance process. 
To address this problem, {\method} deploys a pseudo task to reflect the fairness on the downstream task. Specifically, user communities naturally capture structural and node attribute information, and community labels often correlate strongly with labels from various downstream tasks. Thus, discrimination observed in an unsupervised community detection task can reflect structural biases relevant to downstream tasks. This intuition is further verified by the theoretical analysis. Therefore, community detection is deployed as the pseudo downstream task in {\method} to estimate structural biases without the specific knowledge of downstream tasks. Next, we provide theoretical justification of adopting the pseudo downstream task for fairness on downstream tasks. Then, we present the updated objective function, followed by the design of community detection as the pseudo task.

\vspace{0.1em}
\noindent \textbf{Theoretical Justification}. 
In the following, we present a theorem justifies our motivation that alleviating the bias of downstream tasks can be achieved by reducing the correlation coefficient between the sensitive attribute $\mathcal{S}$ and the community label $\mathbf{C}$. Below, we first present the definition of the Pearson correlation coefficient followed by the theorem and proof.

\begin{myDef}

(Pearson Correlation Coefficient). Pearson correlation coefficient measures the linear correlation between two random variables $X$ and $Y$ as:

\begin{equation}
\begin{aligned}
\rho_{X, Y} = \frac{\mathbb{E}[(X - \mu_X) \cdot (Y - \mu_Y)]}{\sigma_X \cdot \sigma_Y}
\end{aligned} 
\label{eq5}
\end{equation}

\end{myDef}

\begin{theorem}
Let \( C \) and \( S \) represent the community label and sensitive attribute, respectively. Let \( \hat{Y} \) denote the output of a downstream prediction task. Assume that \( C \) is highly correlated with \( \hat{Y} \), i.e., \( \rho_{C, \hat{Y}} \) is larger than a positive constant \( \cos\alpha \). If the model is trained to make \( \rho_{S, C} \) close to zero, i.e., within 
\[
\left[\cos\left(\frac{\pi}{2} + \delta\right), \cos\left(\frac{\pi}{2} - \delta\right)\right],
\]
where \( \delta \) is close to 0, then the correlation between the sensitive attribute \( S \) and the downstream prediction \( \hat{Y} \) satisfies:
\[
\rho_{S, \hat{Y}} \in \left[\cos\left(\frac{\pi}{2} + \delta + \alpha\right), \cos\left(\frac{\pi}{2} - \delta - \alpha\right)\right].
\]

\end{theorem}
This theorem demonstrates that when bias mitigation in pseudo-community detection is achieved through link addition, the upper bound of bias in downstream tasks is consequently constrained.  This theorem validates the deployment of community detection as the pseudo task to facilitate the structural fairness without the knowledge of downstream tasks.

\noindent \textbf{Updated Objective Function with Pseudo Downstream Task}. 
By adopting community detection as the pseudo task, the inner-level optimization is updated to the process of community detection. To measure the bias of the community labels $\mathbf{C}$, we adopt $\Delta_{SP}$ based on the predicted community assignments. Thus, the optimization problem in Eq.(\ref{eq:overall}) can be reformulated as:
\begin{equation}
\begin{aligned}
& \min_{\mathcal{G}'}  \Delta_{SP}({\mathbf{C}}, \mathcal{S}) \\
\text{s.t.} & \quad {\mathbf{C}} = \text{CommunityDetection}(\mathcal{G}'), \\
& \quad \mathcal{E}' \supseteq \mathcal{E}, \quad \mathcal{E}' - \mathcal{E} \leq \Delta.
\end{aligned}
\label{eq:updated}
\end{equation}

\noindent \textbf{Efficient Differentiable Community Detection}. Despite the advantages of using community detection as a pseudo task, traditional community detection methods are typically non-differentiable with respect to the graph structure. 
One may utilize GNN-based community detection $\mathbf{C} = f_{\theta'}(\mathcal{G}')$ with $\theta' = \arg \min_{\theta} \mathcal{L}_{CD}(f_{\theta}(\mathcal{G}'))$, where $\mathcal{L}_{CD}$ is the loss for community detection~\cite{sun2021graph}. However, in this situation, computing the gradient with respect to the graph structure is computationally expensive, as it involves differentiating through the optimization process:
\begin{equation}
    \frac{\partial \mathbf{C}}{\partial \mathcal{G}'} = \frac{\partial f_{\theta'}(\mathcal{G}')}{\partial \mathcal{G}'} + \frac{\partial f_{\theta'}(\mathcal{G}')}{\partial \theta'}\frac{\partial \theta'}{\partial \mathcal{G}'}
    \label{eq:normalCD}
\end{equation}
where the latter term $\frac{\partial \theta'}{\partial \mathcal{G}'}$ requires differentiating through the iterative training procedure of the GNN. 
To overcome this issue, we decouple the community detection into attribute-based clustering and structure-based aggregation.
First, we utilize an MLP-based self-supervised auto-encoder to obtain latent feature representations for nodes. These representations are then clustered via K-means to generate initial community labels:
\begin{equation}
    \mathbf{C}_{\text{init}} = \text{K-means}(\text{MLP}(\mathbf{X})).
    \label{eq:comm_init}
\end{equation}
The number of communities $C$ is the predefined hyperparameter. 

Inspired by label propagation~\cite{zarezadeh2022dpnlp,gasteiger_predict_2019}, we propose to aggregate the initial community labels by graph structure, enabling the incorporation of structural information.  Specifically, the aggregation process for final community labels is described as:
\begin{equation}
\mathbf{C} = \mathrm{softmax}\left( (1-\alpha)^K \hat{\mathbf{A}}^K \mathbf{C}_{\text{init}} + \alpha \sum_{i=0}^{K-1} (1-\alpha)^i \hat{\mathbf{A}}^i \mathbf{C}_{\text{init}} \right),
\label{eq:comm_label}
\end{equation}
where $\hat{A}$ is the symmetric normalized adjacency matrix and $\alpha$ is the probability of restarting from the original node features in aggregation. After $K$-layer aggregation, we get the final community labels. Since label propagation with structures is decoupled with the MLP-based community label initialization, the computation of $\frac{\partial \mathbf{C}}{\mathbf{A}}$ could be  computed without differentiating through the MLP parameters, resulting in much more efficient gradient computation compared to Eq.(\ref{eq:normalCD}).

\subsection{Fairness-Guided Link Addition via Meta Gradients and Gumbel-max Sampling}

In this subsection, we present how to optimize the updated objective function Eq.(\ref{eq:updated}) given the differentiable community detection. Following previous works in updating structures ~\cite{zugner_adversarial_2019}, we solve the bi-level optimization problem in Eq.(\ref{eq:updated}) using meta-gradients. Next, we first derive the computation of meta-gradients. Then, we present a strategy of link addition with meta-gradients that satisfies constraints on link discreteness and the number of new links.

\noindent \textbf{Meta-Gradients of Candidate Links}. The meta-gradient of a candidate link indicates how the addition of this link would influence the fairness of the pseudo community detection task. 
In the following, we formally illustrate the computation of the meta-gradient using the differentiable community detection method:

\begin{equation}
    \nabla_{\mathbf{ A}}^{\text{meta}} = \frac{\partial \Delta_{SP}(\mathbf{C},\mathcal{S})}{\partial \mathbf{C}} \left( \frac{\partial \mathbf{C}}{\partial {\mathbf{A}}} +  \frac{\partial \mathbf{C}}{\partial \mathbf{C}_{\text{init}}} \frac{\partial \mathbf{C}_{\text{init}}}{\partial {\mathbf{A}}} \right)
\end{equation}
Here, the first term represents the direct influence of the perturbed graph structure \(\hat{\mathbf{A}}\) on the final model output, while the second term accounts for the indirect effects propagated through initialized community labels. Note that initialized community labels are independent to the graph structure. Therefore, we can deduce that the indirect term $\frac{\partial \mathbf{C}_{\text{init}}}{\partial \mathbf{A}} = 0$. So finally the meta-gradient consequently be simplified as: 

\begin{equation}
    \nabla_{\mathbf{ A}}^{\text{meta}} = \frac{\partial \Delta_{SP}(\mathbf{C},\mathcal{S})}{\partial \mathbf{C}} \frac{\partial \mathbf{C}}{\partial {\mathbf{A}}}.
    \label{eq:meta}
\end{equation}
With the Eq.(\ref{eq:meta}), the meta-gradients of candidate links can be efficiently computed. 
Next,  we'll introduce how we optimize the adjacent matrix \({\mathbf{A}}\) discretely to satisfy the constraints of our optimization problem .

\noindent \textbf{Gumbel-max Sampling for Link Addition.} With the Eq.(\ref{eq:meta}), we get  meta-gradients of graph structures, i.e., $\nabla_{{\mathbf{A}}}^{\text{meta}}$.  And based on meta-gradient, a straightforward way to optimize the graph structure is:
\begin{align}
 \mathbf{{A}^\prime} = \mathbf{A} - \alpha 
 \nabla_{{\mathbf{A}}}^{\text{meta}} .
\end{align}
However, this approach cannot guarantee the constraints of link discreteness and the maximum number of link addition. To overcome these issues, we reformulate edges adding process as a stochastic sampling process based on reweighted meta-gradients. According to the rule for the derivative of a scalar with respect to a matrix, the final gradient matrix $\nabla_{{\mathbf{A}}}^{\text{meta}}$ has the same shape as the adjacency matrix $\mathbf{A}$ , and $\nabla_{{\mathbf{A_{i,j}}}}^{\text{meta}}$ represents the possible effect of the edge i,j on fairness. Intuitively, to grow an unbiased community, more group-cross edges are needed to add, so we amplify gradients for edges connecting nodes with differing sensitive attributes while enabling probabilistic batch selection. Specifically, for each candidate edge $(i, j) \in \mathcal{E} = \left\{ (u, v) \mid \mathbf{A}_{u,v} = 0 \right\}$, we compute an adjusted gradient score:
\begin{equation}
\tilde{\nabla}{i,j} = -{\nabla_{{\mathbf{A_{i,j}}}}^{\text{meta}}} \cdot {(1 + \beta \cdot \mathbb{I}(s_i \neq s_j))},
\end{equation}
Here $\beta$ controls inter-group connection incentives, and taking a negative value of the gradient means potential edges with smaller gradients have higher scores. We then sample edges via Gumbel-softmax reparameterization:
\begin{equation}
P(i,j) = \frac{\log(\tilde{\nabla}{i,j} + \epsilon) + g_{i,j}}{\tau}, \quad g_{i,j} \sim \text{Gumbel}(0,1),
\end{equation}
where $\ g_{i,j}$ is the gumbel noise, $\tau$ modulates exploration-exploitation trade-offs and $\epsilon$ prevents numerical underflow. And simultaneously the top-k edges with highest scores are added :

\begin{equation}
{\mathbf{A}}^\prime = \mathbf{A} + \sum_{(i,j)\in \text{Top}k({P(i,j)})} \delta_{i,j}.
\end{equation}
These formulas describe the whole meta-gradient optimization process for dynamically updating the graph structure to optimize fairness. After one round of the process of sampling edges to update the graph structure, the whole optimization pipeline goes back to aggregation step. Repeating the cycle for several times, we get the final structure ${\mathbf{A}}^\prime$ with new fair links. The detailed algorithm of {\method} can be found in Appendix~\ref{app:algorithm}.

\subsection{Time Complexity Analysis}
In this section, we give the time complexity for adding a link for a specific node $v$. Specifically, the computational complexity of {\method} is determined by two components, i.e., differentiable community detection, and fairness-guided link addition via meta
gradients. To recommend a link for node $v$, the differentiable community detection involve the existing links and candidate links in the differential aggregation phase. The time complexity would be $O((d+1)c|\mathcal{V}|)$, where $d$ and $c$ denotes the average degree of graph and predefined community number. According to Eq.(\ref{eq:meta}), the cost of meta-gradient computation is the same as the forward computation $O((d+1)c|\mathcal{V}|)$. Finally, the selection of optimal link would be $O(|\mathcal{V}|\log{|\mathcal{V}|})$.   
Therefore, for a given node $v$, the total time complexity of suggesting a link to add would be $O(2(d+1)c|\mathcal{V}|+|\mathcal{V}|\log{|\mathcal{V}|})$. In comparison, a GNN-based link recommendation has the time complexity as $O(dh|\mathcal{V}|+|\mathcal{V}|\log{|\mathcal{V}|})$, where $h$ indicates the hidden dimension. Thus, the computational cost of {\method} is similar to standard link recommendation methods. Additionally, the actual running time of adding a single link is reported in the Tab.~\ref{tab:Running_time} in the Appendix, which is less than 0.02 seconds.

\section{Experiments}

\definecolor{lightgray}{gray}{0.85}
\newcommand{\secondbest}[1]{\cellcolor{lightgray}#1}
\newcommand{\best}[1]{\textbf{#1}}

In this section, we conduct experiments on real-world user graphs to answer the following research questions:
\begin{itemize}[leftmargin=*]
    \item \textbf{RQ1}:Is {\method} capable of adding new links that can effectively promote fairness in user graph structures?
    \item \textbf{RQ2}: How does the number of new links introduced for fairness guidance influence the fairness of GNN-based applications?
    \item \textbf{RQ3}: How does the pseudo downstream task and dynamic link addition by meta-gradients contribute to {\method} in guiding the fairness of graph structures ?
\end{itemize}

\subsection{Experimental Settings}
\label{sec:set_up}

\subsubsection{Datasets}
Three real-world include Pokec-n, Pokec-z and Github are adopted for our experiments. The Github dataset is introduced in Sec. 3.2. The Pokec-z/n datasets~\cite{dai2021say} are both sampled from Pokec, which is the most popular social network in Slovakia. Node features include gender, age, hobbies, regions, and other attributes. Edges in the dataset represent friendship relationships among users. The statistics of three datasets are in the table~\ref{tab:dataset_stats}. We set the labels larger than 1 as 1 and select 25\% of labeled nodes as the validation set and 25\%  labeled nodes as the test set. We also select 7000 ground-truth labels as training set for Github dataset and 500 labels for Pokec datasets following the prior works.

\subsubsection{Baselines} Since our proposed framework guides graph fairness via link addition, we compare it with three pre-processing methods that could debias the graph structures via injecting new links. Additionally, we include two simple link-addition strategies, i.e., {Rand. Add} and {Link Pred.}, for comparisons:

\begin{itemize}[leftmargin=*]
    \item \textbf{Rand. Add}: This baseline randomly add links to the graph.
    \item \textbf{Link Pred.} ~\cite{schlichtkrull2018modeling}: It trains a GCN-based graph autoencoder. Then,  pairwise cosine similarity between node embeddings are used for the link prediction. 
    \item \textbf{EDITS} ~\cite{dong2022edits}: A preprocessing framework designed to mitigate bias in attributed networks by optimizing bias metrics across both attribute and structural modalities. We adopt its structure debiasing part and restrict it to only add edges according to the trained edge score.
    \item \textbf{Fairgen} ~\cite{zheng2024fairgen}: A generative model that promotes fairness in graph generation by incorporating label information and fairness constraints, thereby reducing representation disparity. We sample new edges in the generated graph structure and combine them with the original graph structure.
    \item \textbf{Graphair} ~\cite{ling2023learning}: A method for learning fair graph representations through the automatic discovery of fairness-aware augmentations.We choose its structure augmentation part and ensure the preservation of original edges in the augmenting process. 
\end{itemize}

\begin{table}[t]
  \centering
  \caption{The statistics of the three real-world datasets. }
  \vskip -0.5em
  \label{tab:dataset_stats}
  \small 
  \begin{tabular}{lcccl}
    \toprule
    {Dataset} & {\# Nodes} & {\# Edges} & {Sensitive attribute} & {Label} \\
    \midrule
    Github      & 32,132  & 270,088  & Country   & Popularity \\
    Pokec-z     & 67,797  & 882,765 & Region & Job Field \\
    Pokec-n     & 66,569  & 729,129 & Region & Job Field \\
    \bottomrule
  \end{tabular}
  \vskip -1em
\end{table}

\subsubsection{Implementation Details} {\method} is implemented using Pytorch and optimized via Adam optimizer ~\cite{kingma2014adam}. Each iteration we fix the adding number as 100.
For each method, we conduct experiments with seed \{10,20,30,40,50\} and compute the mean and standard deviations of F1, AUC, $\Delta_{SP}$ and $\Delta_{EO}$. All methods consistently add the same number of new edges, corresponding to 3\% of the original graph's total edge number for Pokec-n, 1.5\% for Pokec-z and 4\% for Github.  For GNN backbone, we employ 2-layer GCN, 2-layer GraphSage and 10-layer APPNP. The hidden dimension is set as 128,  the learning rate and training epochs are set as $1 \times 10^{-3}$ and 1000 respectively.
\begin{table*}[t]
\centering
\caption{Results for the node classification on GCN model, with best results in bold and runner-up results in gray.}
\label{tab:node}
\vskip -0.5em 
\begin{tabularx}{0.95\linewidth}{ll | CCC | CCCC}
\toprule

{Dataset} & {Metrics (\%)} & {Vanilla} &  {Rand. Add} & {Link Pred.}  &{Edits} & {Graphair} & {Fairgen} & {\method} \\ 
\midrule

\multirow{4}{*}{Github} & F1 (\textbf{↑})   & 78.6 $\pm$ 0.2  & 78.0 $\pm$ 0.1 & 78.8 $\pm$ 0.1 & \textbf{78.6 $\pm$ 0.1}  &  77.5 $\pm$ 0.2 & 77.5 $\pm$ 0.2 & \secondbest {77.8 $\pm$ 0.1}  \\
                         & AUC (\textbf{↑})  & 85.4 $\pm$ 0.5 & 84.3 $\pm$ 0.1  & 85.4 $\pm$ 0.1 & \textbf{84.8 $\pm$ 0.3} & 84.2 $\pm$ 0.1 & 84.0 $\pm$ 0.1  &  \secondbest {84.3 $\pm$ 0.1} \\
                         & $\Delta_{SP}$ (\textbf{↓}) & 12.5 $\pm$ 0.4 & 11.0 $\pm$ 0.2 & 11.9 $\pm$ 0.2 & 11.5 $\pm$ 0.8  & 11.1 $\pm$ 0.3 & \secondbest{10.8 $\pm$ 0.3} & \textbf{\phantom{0}8.6 $\pm$ 0.2} \\
                         & $\Delta_{EO}$ (\textbf{↓}) & \phantom{0}8.5 $\pm$ 0.5 & \phantom{0}8.3 $\pm$ 0.3 & \phantom{0}8.3 $\pm$ 0.4 & 8.5 $\pm$ 0.8 & \phantom{0}8.4 $\pm$ 0.3 & \phantom{0}\secondbest{8.3 $\pm$ 0.1} & \textbf{\phantom{0}6.0 $\pm$ 0.2} \\
\midrule

\multirow{4}{*}{pokec-n} & F1 (\textbf{↑})   &  66.8 $\pm$ 0.7 & 67.1 $\pm$ 0.3 & 66.7 $\pm$ 0.5 & 65.4 $\pm$ 0.4 & \secondbest{65.4 $\pm$ 0.3} & 65.2 $\pm$ 0.3  & \textbf{66.2 $\pm$ 0.3} \\
                         & AUC (\textbf{↑})  & 75.4 $\pm$ 0.1 & 74.8 $\pm$ 0.1 & 75.4 $\pm$ 0.2 & 72.8 $\pm$ 0.6  & \secondbest{74.3 $\pm$ 0.1} & 74.2 $\pm$ 0.2 & \textbf{74.9 $\pm$ 0.2}  \\
                         & $\Delta_{SP}$ (\textbf{↓})  & \phantom{0}8.5 $\pm$ 0.7 & \phantom{0}9.2 $\pm$ 1.0 & \phantom{0}7.9 $\pm$ 0.8 & \phantom{0}\secondbest{3.7 $\pm$ 1.3} & \phantom{0}9.7 $\pm$ 0.4 & \phantom{0}9.6 $\pm$ 0.7  & \textbf{\phantom{0}1.3 $\pm$ 0.8} \\
                         & $\Delta_{EO}$ (\textbf{↓})  & 11.9 $\pm$ 1.4 & \phantom{0}12.3 $\pm$ 1.0 & 11.0 $\pm$ 1.6 & \phantom{0}\secondbest{6.3 $\pm$ 1.5} & \phantom{0}11.7 $\pm$ 1.6 & \phantom{0}11.7 $\pm$ 0.5  & \textbf{\phantom{0}3.1 $\pm$ 0.5} \\
\midrule

\multirow{4}{*}{pokec-z} & F1 (\textbf{↑})   & 70.6 $\pm$ 0.4 & 70.5 $\pm$ 0.2 & 70.3 $\pm$ 0.6 & 67.1 $\pm$ 1.1  & 69.0 $\pm$ 0.5 & \secondbest{69.1 $\pm$ 0.9} & \textbf{70.2 $\pm$ 0.4}\\
                         & AUC (\textbf{↑})  & 77.2 $\pm$ 0.1 & 76.8 $\pm$ 0.1 & 77.2 $\pm$ 0.1 & 75.5 $\pm$ 0.7 & \secondbest{76.0 $\pm$ 0.1} & 76.0 $\pm$ 0.3 & \textbf{76.1 $\pm$ 0.1} \\
                         & $\Delta_{SP}$ (\textbf{↓}) & \phantom{0}9.1 $\pm$ 0.9 & \phantom{0}8.8 $\pm$ 0.8 & \phantom{0}9.5 $\pm$ 0.8 & \phantom{0}\secondbest{5.0 $\pm$ 2.0} & \phantom{0}6.4 $\pm$ 1.5 & \phantom{0}6.9 $\pm$ 2.4  & \textbf{\phantom{0}3.1 $\pm$ 0.7} \\
                         & $\Delta_{EO}$ (\textbf{↓}) & \phantom{0}8.2 $\pm$ 1.2 & \phantom{0}8.0 $\pm$ 0.7 & \phantom{0}8.1 $\pm$ 1.1  & \phantom{0}\secondbest{5.1 $\pm$ 1.8} & \phantom{0}5.5 $\pm$ 2.0 &  \phantom{0}7.6 $\pm$ 2.6  &  \textbf{\phantom{0}4.6 $\pm$ 0.4}\\
\bottomrule
\end{tabularx}
\end{table*}

\begin{table*}[t!]
\centering
\caption{Results on the downstream community detection task, with best results in bold and runner-up results in gray.}
\vspace{-1em}
\begin{tabularx}{0.95\linewidth}{ll|CCC|CCCC}
\toprule
{Dataset} & {Metrics (\%)} & {Vanilla} & {Rand. Add}  & {Link Pred.} & {Edits} & {Graphair} & {Fairgen} & {\method} \\ 
\midrule

\multirow{1}{*}{Github} 
                         & $\Delta_{SP}$ (\textbf{↓}) & 41.9 $\pm$ 1.5 & 35.7 $\pm$ 2.8 &
                         38.2 $\pm$ 4.0 & 40.2 $\pm$ 3.6 & \secondbest{31.9 $\pm$ 2.1} & 36.3 $\pm$ 3.8  & \textbf{29.0 $\pm$ 2.2} \\

\multirow{1}{*}{pokec-n} 
                         & $\Delta_{SP}$ (\textbf{↓}) & 83.2 $\pm$ 3.0 &  76.1 $\pm$ 0.8 & 77.5 $\pm$ 2.6 & 79.6 $\pm$ 3.3  & \secondbest{79.2 $\pm$ 2.9} & 80.0 $\pm$ 1.4 & \textbf{74.1 $\pm $ 3.4}\\
                         
\multirow{1}{*}{pokec-z} 
                         & $\Delta_{SP}$ (\textbf{↓}) & 77.7 $\pm$ 0.9 & 79.3 $\pm$ 3.7 & 80.4 $\pm$ 3.0 & 82.3 $\pm$ 3.8 & 81.7 $\pm$ 1.7 & \secondbest{79.6 $\pm $ 0.7} & \textbf{70.1 $\pm$ 4.0} \\
\bottomrule
\end{tabularx}
\label{tab:community}
\end{table*}

\subsection{Results of Fairness Guidance}
To answer \textbf{RQ1}, we first demonstrate that the links suggested by {\method} effectively improve the fairness of various GNN models across different downstream tasks. Then, we discuss the trade-off between fairness and utility in user graphs.

\noindent \textbf{Fairness Improvements on Downstream Tasks}. 
Two types of downstream tasks, i.e., node classification and community detection, are utilized to evaluate the effectiveness of {\method} in guiding the user graph structures toward fairness. 
For node classification task, a GCN is trained on these updated graph structures as the downstream classifier. For community detection task, we adopt the Louvain method ~\cite{blondel2008fast} based on modularity optimization.  We utilize the  $\Delta_{SP}$ and $\Delta_{EO}$ as metrics to evaluate the fairness of the GCN-based downstream classifiers. Fairer results in downstream tasks indicate a better performance in guiding the graph structure towards fairness. The results on two downstream tasks are presented in Tab.~\ref{tab:node} and Tab.~\ref{tab:community}. From these tables, we observe that:

\begin{itemize}[leftmargin=*] 
    \item  For the node classification task,  all baselines perform poor in fairness metric. In contrast, models trained on graphs with new links added by {\method} achieves great performance in fairness  and only sacrifices little task performance. 
    \item For the community detection task, the Rand. Add baseline shows limited effectiveness in enhancing fairness. Furthermore, links added by some baseline methods even amplify biases between communities. By contrast, fairness-guided links added by {\method} could bring significant improvements in community detection fairness compared with baselines.
\end{itemize}
\begin{figure}[t]
  \centering
  \begin{subfigure}{0.48\linewidth}  
    \includegraphics[width=\textwidth]{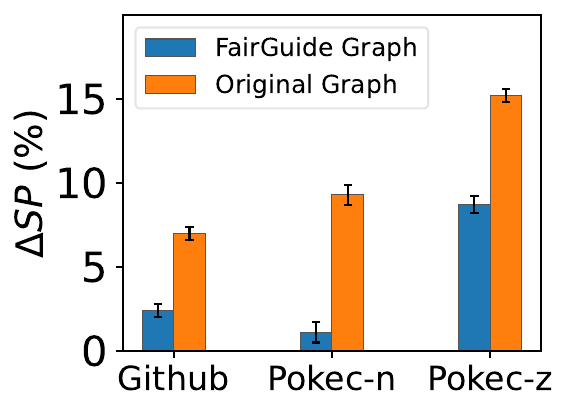}  
    \vspace{-0.5em}
    \caption{APPNP}
    \label{fig:abl_sp}
  \end{subfigure}
  \begin{subfigure}{0.48\linewidth}
    \includegraphics[width=\textwidth]{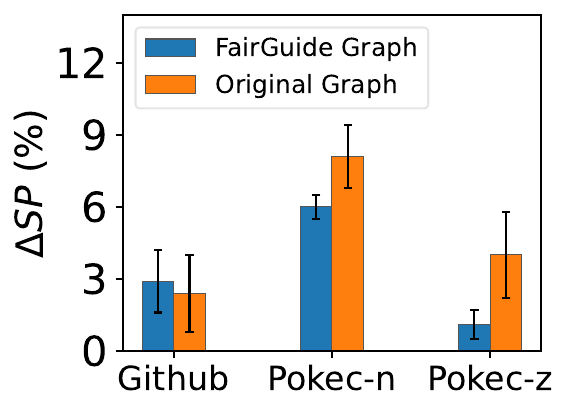} 
    \vspace{-0.5em}
    \caption{GraphSage}
    \label{fig:abl_f1}
  \end{subfigure}
  \vskip -1em
  \caption{Fairness improvements on different GNNs.}
  \vskip -0.8em
  \label{fig:backbone}
\end{figure}
\noindent \textbf{Generalization to Various GNN Architectures}. In addition, to further verify the structural fairness of graphs guided by {\method}, we  evaluate the impact of fairness-guided links on two different GNN backbones, i.e. , GraphSage~\cite{hamilton2017inductive} and APPNP~\cite{gasteiger_predict_2019}. Results in fairness are presented in Fig.~\ref{fig:backbone}. From Fig~\ref{fig:backbone}, we observe that different GNN backbones consistently achieve fairness improvements when trained on graphs guided by {\method}. This demonstrates that {\method} can grow fair graph structures that facilitate different GNN models.

\vspace{0.3em}
\noindent \textbf{Trade-off between Utility and Fairness}.   We visualize the F1-$\Delta_{SP}$ of different methods by varying the strength of fairness guidance. Specifically, we vary the percentage of adding fairness-guided links from 0.5\% to 3\%. The results are given in Fig.~\ref{fig:tradeoff}. 
Scattered points in various colors represent results from different methods. We further fit a straight line for each method. From these two figures, we observe that great improvements fairness are always accompanied with a decrease in task performance. And compared to other baseline models, {\method} achieves a better utility-fairness balance. {\method} either have better task performance at the same fairness level or reduce more bias at equivalent task performance.

\begin{figure}[t]
  \centering
  \begin{subfigure}[b]{0.48\linewidth}  
    \includegraphics[width=\textwidth]{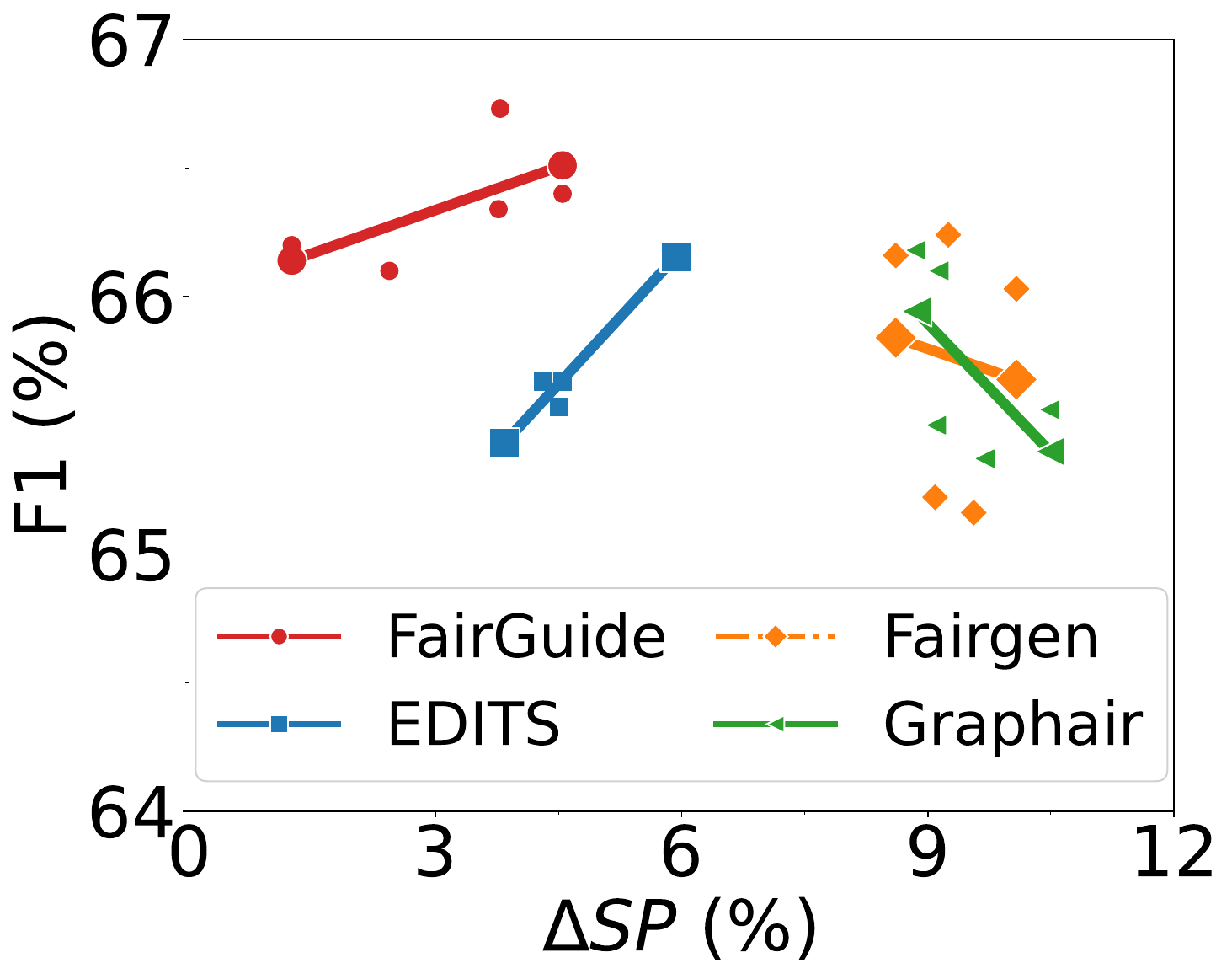}  
    \caption{Pokec-n}
    \label{fig:poken_add}
  \end{subfigure}
  \begin{subfigure}[b]{0.48\linewidth}
    \includegraphics[width=\textwidth]
    {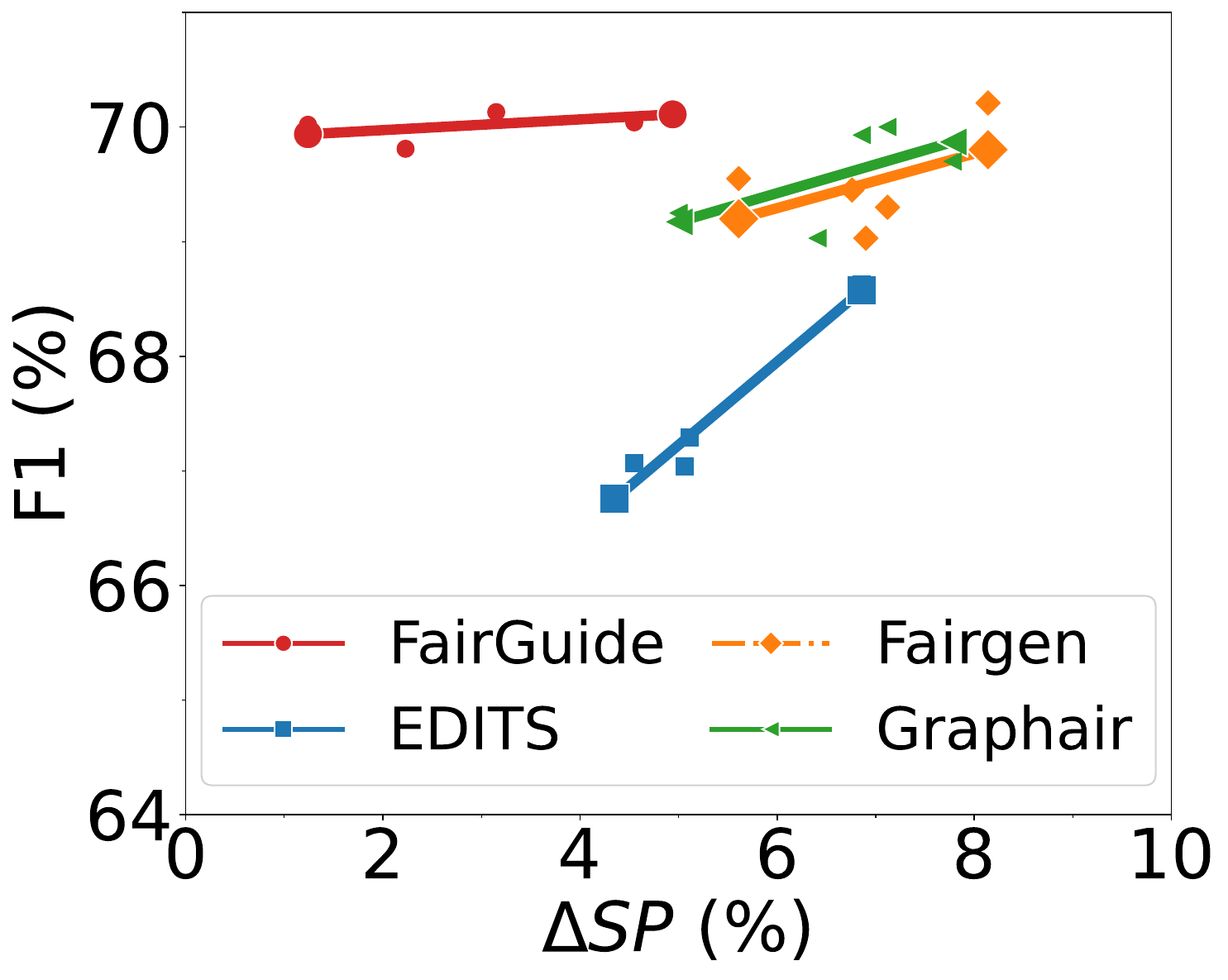}  
    \caption{Pokec-z}
    \label{fig:pokez_trend}
  \end{subfigure}
  \vskip -1em
  \caption{Visualization of trade-off between utility and fairness. Methods in the upper-left region are better.}
  \label{fig:tradeoff}
\end{figure}

\subsection{Impacts of the Number of New Links}
To answer the \textbf{RQ2},  we we vary edge adding percentage  based on the existing edges from 0.5\% to 3\% to investigate the impact of link addition size. We conduct experiments for node classification task based on GCN. The impacts of link addition rate on pokec-n and pokec-z are shown in Fig.~\ref{fig:rate_sensitivity}. We have similar observations on other datasets. From these results, we  can observe that with {\method}, the fairness of the downstream GNN classifier improves as the percentage of added edges increases. During this process, {\method} consistently outperforms baseline methods in terms of fairness while preserving utility, which demonstrates the effectiveness of {\method} in guiding graph structures toward fairness.

\begin{figure}[t]
  \centering
  \begin{subfigure}{0.48\linewidth}  
    \includegraphics[width=\textwidth]{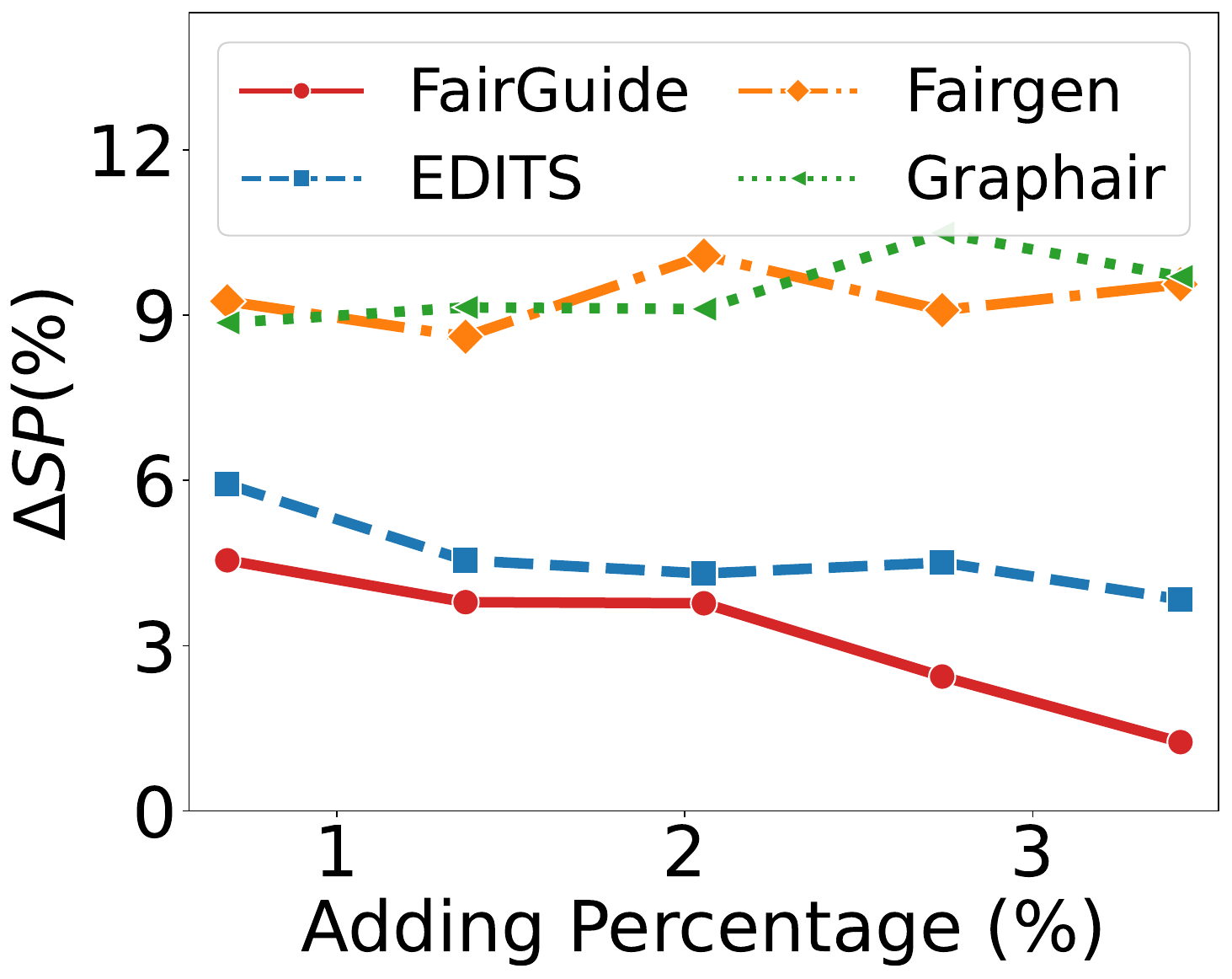}  
    \vspace{-1.8em}
    \caption{$\Delta_{SP}$(\%) on Pokec-n}
    \vspace{0.5em}
  \end{subfigure}
  \begin{subfigure}{0.48\linewidth}
    \includegraphics[width=\textwidth]
    {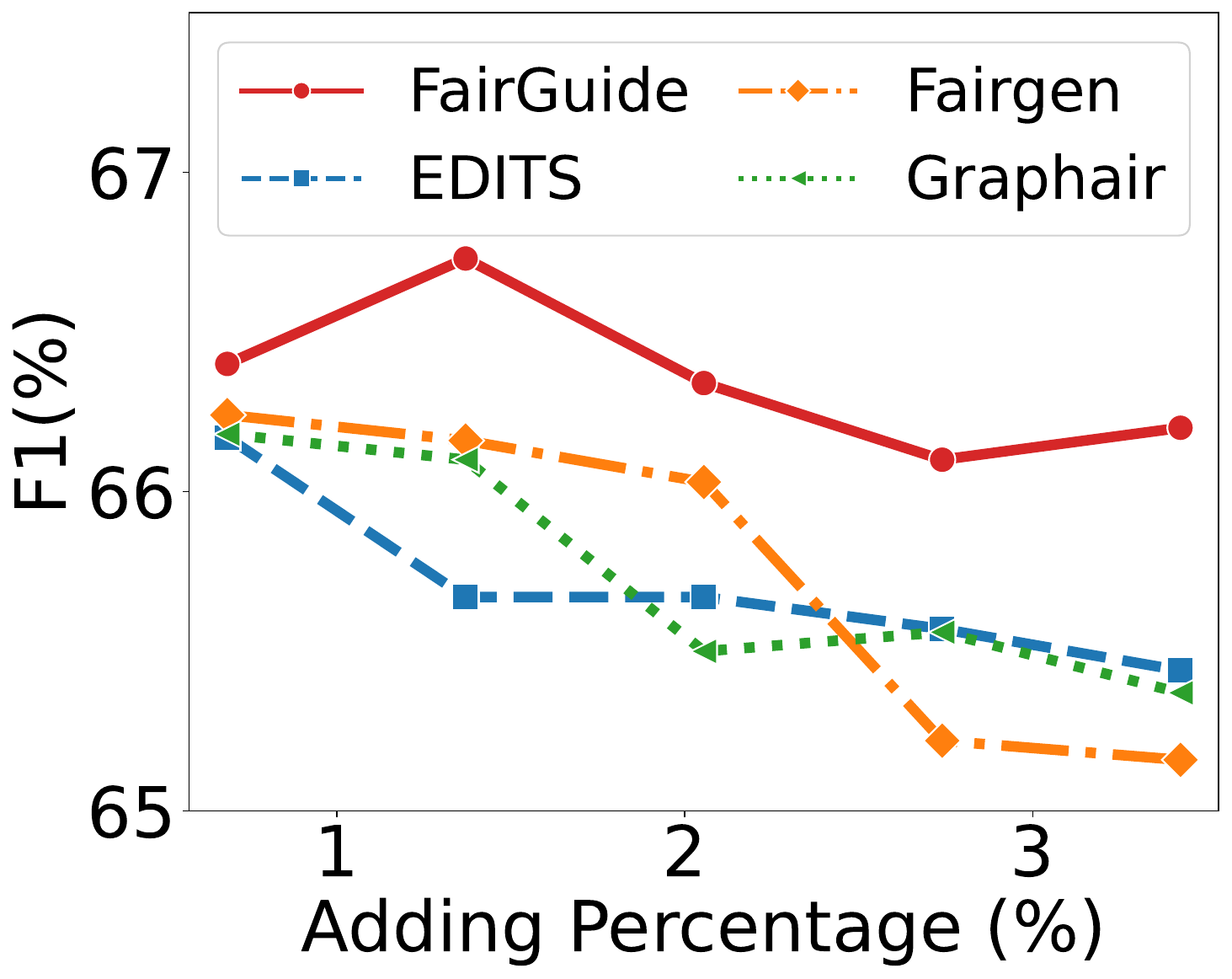}  
    \vspace{-1.8em}
    \caption{F1(\%) on Pokec-n}
    \vspace{0.5em}
  \end{subfigure}
  \begin{subfigure}{0.48\linewidth}  
    \includegraphics[width=\textwidth]{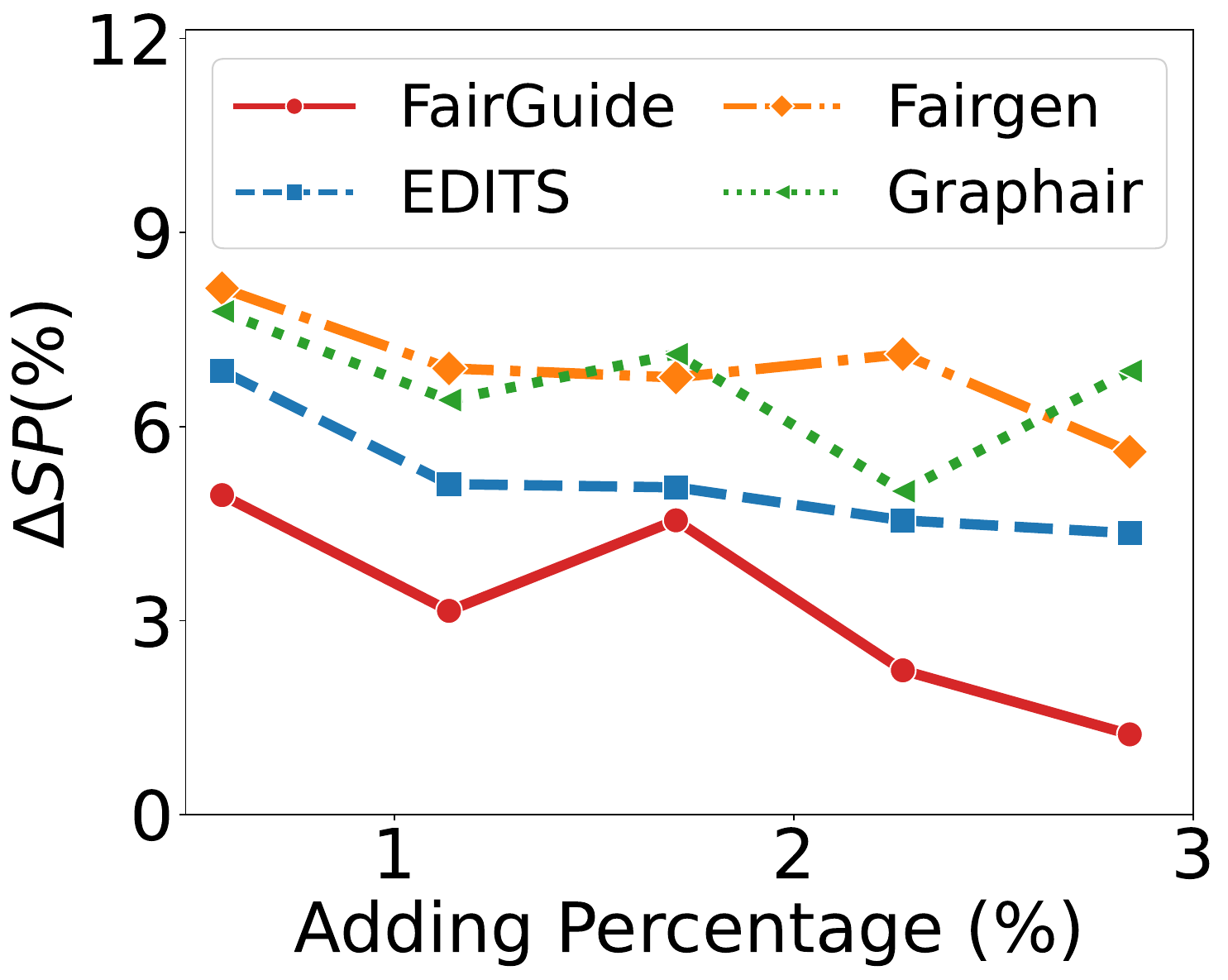}  
    \vspace{-1.8em}
    \caption{$\Delta_{SP}$(\%) on Pokec-z}
  \end{subfigure}
  \begin{subfigure}{0.48\linewidth}
    \includegraphics[width=\textwidth]
    {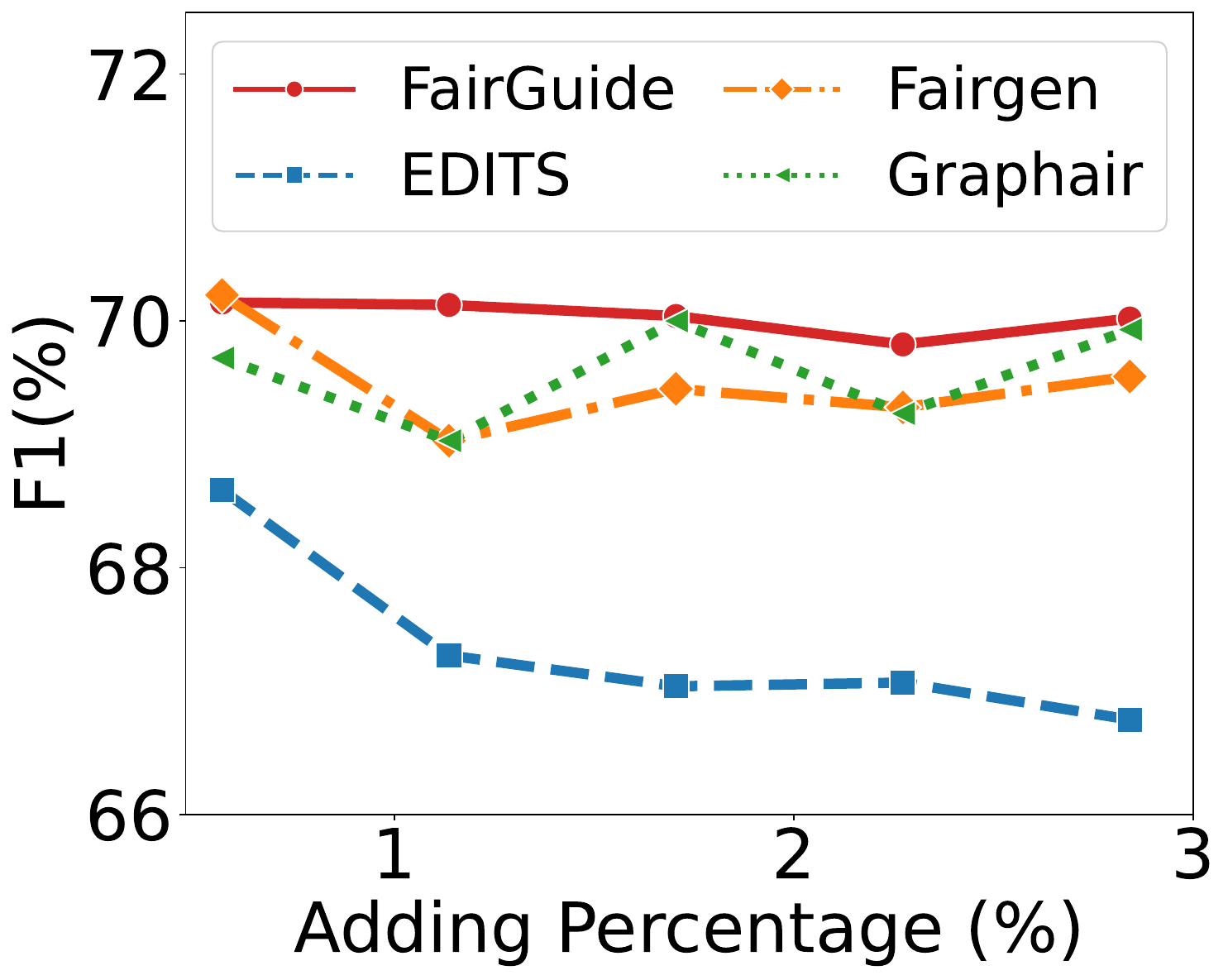}  
    \vspace{-1.8em}
    \caption{F1(\%) on Pokec-z}
  \end{subfigure}
  \vskip -1em
  \caption{Impacts of number of added links.}
  \vskip -1em
  \label{fig:rate_sensitivity}
\end{figure}

\begin{figure}[b]
  \centering
  \begin{subfigure}[b]{0.48\linewidth}  
    \includegraphics[width=\textwidth]{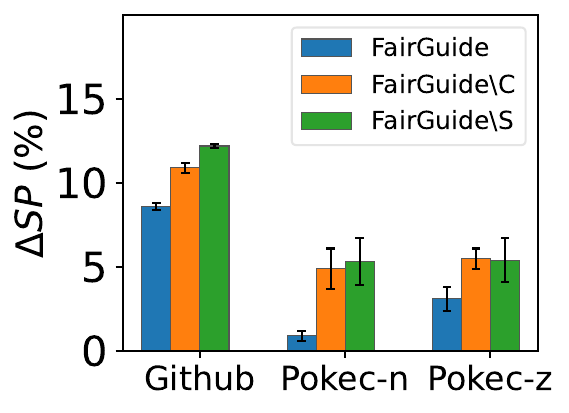} 
    \vspace{-1em}
    \caption{$\Delta$SP (\%)}
    \label{fig:abl_sp}
  \end{subfigure}
  \begin{subfigure}[b]{0.48\linewidth}
    \includegraphics[width=\textwidth]{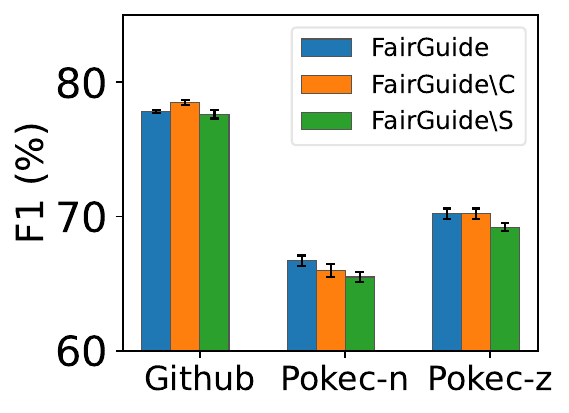}  
    \vspace{-1em}
    \caption{F1 (\%)}
    \label{fig:abl_f1}
  \end{subfigure}
  \vskip -1em
  \caption{Ablation study on three real-world user graphs.}
  \label{fig:ablation}
\end{figure}

\subsection{Ablation Study}
{\method} relies on two critical components to effectively guide the fairness of graph structures: the pseudo downstream task and link addition with Gumbel-max sampling. 
To assess the impact of these components to answer \textbf{RQ3}, we conduct ablation studies using the GCN-based node classification task. Specifically, to demonstrate the importance of employing community detection as the pseudo downstream task, we compare with a variant named {\method}$\backslash$C where the initial community labels are replaced with randomly generated labels.
Additionally, to evaluate the effectiveness of Gumbel-max sampling in the fairness-guided link addition , we introduce another variant, {\method}$\backslash$S, in which links are determined through a single iteration without dynamic updates.  The results  are shown in Fig.~\ref{fig:ablation}. From these results, we observe:
\begin{itemize}[leftmargin=*]
    \item {\method} achieves significantly smaller $\Delta_{SP}$ compared to the variant {\method}$\backslash$C, where the community detection labels are replaced by propagated random labels. This is because community labels naturally correlate with labels in various downstream tasks. Thus, incorporating community detection as a pseudo task effectively enhances fairness in downstream tasks, which empirically verifies our theoretical analysis.
    \item  {\method}$\backslash$S consistently underperforms {\method} in fairness evaluations across all datasets, highlighting the importance of dynamic link addition via Gumbel-max sampling.
\end{itemize}

\subsection{Hyperparameter Analysis}
In {\method}, the cross-group boost rate $\beta$ and the number of community clusters C are two key hyperparameters.To investigate how the two hyperparameters affect fairness and task performance, We conduct GCN node classification task on the Pokec-n. Similar trends are also observed in other datasets. We vary $\beta$ to \{0.01,0.1,1.0,10.0,100.0\} and $C$ to \{5,10,15,20,25\}, respectively. All other settings follow the description given in Sec.~\ref{sec:set_up}.
As shown in Fig.~\ref{fig:sensitivity}, the task performance F1 score remains stable across the tested hyperparameter ranges, fluctuating by only around 1–2 percentage points. For the fairness performance, {\method} achieves a strong debiasing effect when $C \ge 15$ and $\beta \ge 1.0 $. A larger number of clusters $C$ allows the model to capture more fine-grained feature information, and a large cross-group boost rate $\beta$ effectively establishes connections between different sensitive groups, which will benefit the graph structural fairness.

\begin{figure}[t]
  \centering
  \begin{subfigure}{0.48\linewidth}  
    \includegraphics[width=\textwidth]{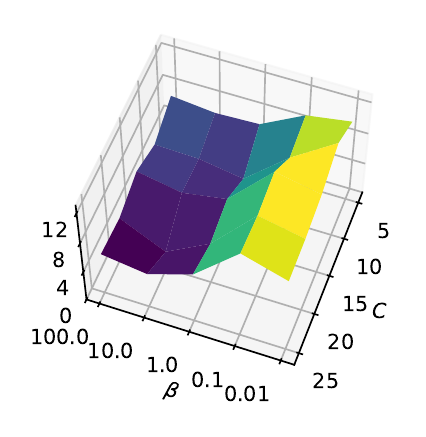}  
    \vspace{-1.5em}
    \caption{$\Delta_{SP}$ (\%)}
    \label{fig:sp}
  \end{subfigure}
  \begin{subfigure}{0.48\linewidth}
    \includegraphics[width=\textwidth]{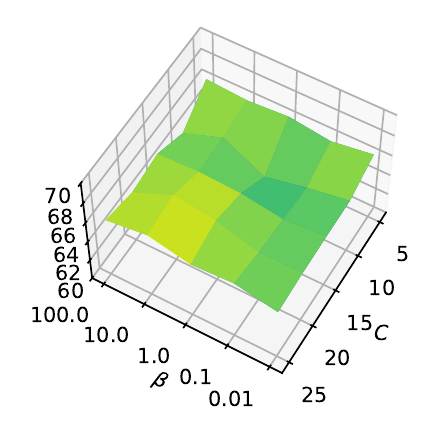}  
    \vspace{-1.5em}
    \caption{F1 score (\%)}
    \label{fig:auc}
  \end{subfigure}
  \vskip -0.5em
  \caption{Hyperparameter sensitivity analysis.}
  \label{fig:sensitivity}
\end{figure}

\section{Conclusion and Future Work}
In this work, we study a novel problem of fairness guidance via link addition which aims to guide existing graph structures toward fairness. To solve this problem, we form it as bi-level optimization problem and propose an algorithm called {\method} which utilizes the  differentiable pseudo community detection task to optimize the graph structure. Specifically, we theoretically demonstrate generalize fairness can be achieved by reducing the bias in pseudo task via new links.  To efficiently identify beneficial new links, we compute meta-gradients with respect to the graph structure and adopt an edge-sampling technique to discretely update the structure. Experiment results on real-world datasets show the effectiveness of our proposed algorithm in solving the fair guidance problem. For the future work, there are several directions to investigate. First, in this paper links are added just according to the fairness without consideration of enhancing the structure utility . In future we will explore how to add new links to simultaneously improve the task performance and fairness. Second, solving various sensitive attribute by adding links is also a interesting problem to investigate. We plan to extend the FairGuide to achieve more general fairness on diverse sensitive attributes.

\appendix

\section{Proof of Theorem 1}

\begin{proof}
To prove this theorem, we first introduce the following lemmas and theorem:

\begin{lemma}
Given a unit sphere centered at origin \( O \), let \( A, B, C \) be three points on the surface. Assume angles \( AOB = \theta_1 \) and \( BOC = \theta_2 \). Then, the cosine of angle \( AOC \) satisfies:
\[
\cos(AOC) \in [\cos(\theta_1 + \theta_2), \cos(\theta_1 - \theta_2)].
\]
\end{lemma}

\begin{proof}
From the spherical law of cosines:
\[
\cos\theta_3 = \cos\theta_1\cos\theta_2 + \sin\theta_1\sin\theta_2\cos B',
\]
where \( B' \) is the angle opposite to \( B \) in the spherical triangle \( ABC \). Since angles are in \([0, \pi]\), we have:
\begin{align*}
\cos\theta_3 &\geq \cos\theta_1\cos\theta_2 - \sin\theta_1\sin\theta_2 = \cos(\theta_1 + \theta_2), \\
\cos\theta_3 &\leq \cos\theta_1\cos\theta_2 + \sin\theta_1\sin\theta_2 = \cos(\theta_1 - \theta_2).
\end{align*}
Thus, the lemma is proven.
\end{proof}

\begin{lemma}
Given two random variables \( X \) and \( Y \), the Pearson correlation coefficient is equivalent to the cosine similarity between their z-score vectors \( x' \) and \( y' \), where:
\[
x'_i = \frac{X_i - \mu_X}{\sigma_X}, \quad y'_i = \frac{Y_i - \mu_Y}{\sigma_Y}.
\]
\end{lemma}

\begin{proof}
By definition, the Pearson correlation coefficient is:
\[
\rho_{X,Y} = \frac{\mathbb{E}[(X - \mu_X)(Y - \mu_Y)]}{\sigma_X \sigma_Y} = \lim_{n \to \infty} \sum_{i=1}^{n} x'_i y'_i = \cos(x', y').
\]
Thus, the lemma is proven.
\end{proof}

\begin{theorem}
Given three random variables \( X, Y, Z \) with correlation coefficients \( \rho_{X,Y} = \cos\alpha \) and \( \rho_{Y,Z} = \cos\beta \), then:
\[
\rho_{X,Z} \in [\cos(\alpha + \beta), \cos(\alpha - \beta)].
\]
\end{theorem}

\begin{proof}
From Lemma 2, the correlation coefficients correspond to cosine similarities between z-score vectors. Applying Lemma 1 with these vectors yields:
\[
\rho_{X,Z} = \cos(x', z') \in [\cos(\alpha + \beta), \cos(\alpha - \beta)].
\]
Thus, the theorem is proven.
\end{proof}

Now, applying Theorem 2 directly, we have:

Given \(\rho_{C, \hat{Y}} = \cos\alpha\), and \(\rho_{S,C}\) within \([\frac{\pi}{2}+\delta, \frac{\pi}{2}-\delta]\), we obtain:
\[
\rho_{S, \hat{Y}} \in \left[\cos\left(\frac{\pi}{2} + \delta + \alpha\right), \cos\left(\frac{\pi}{2} - \delta - \alpha\right)\right].
\]
Thus, the theorem is proven.
\end{proof}

\section{Case Study}
To further investigate how our method impacts real-world communities, we provide one case study in Github dataset. Specifically, we examine the percentage of connection between Bangladeshi developers and other countries developers, which is shown in Fig. ~\ref{fig:analysis}. 

After applying {\method}, we observe increased connectivity between developers from Bangladesh and those developed countries. These results demonstrate the practical effectiveness of our method in guiding the growing of community toward fairness.

\begin{figure}[t]
  \centering
  \includegraphics[width=0.7\linewidth]{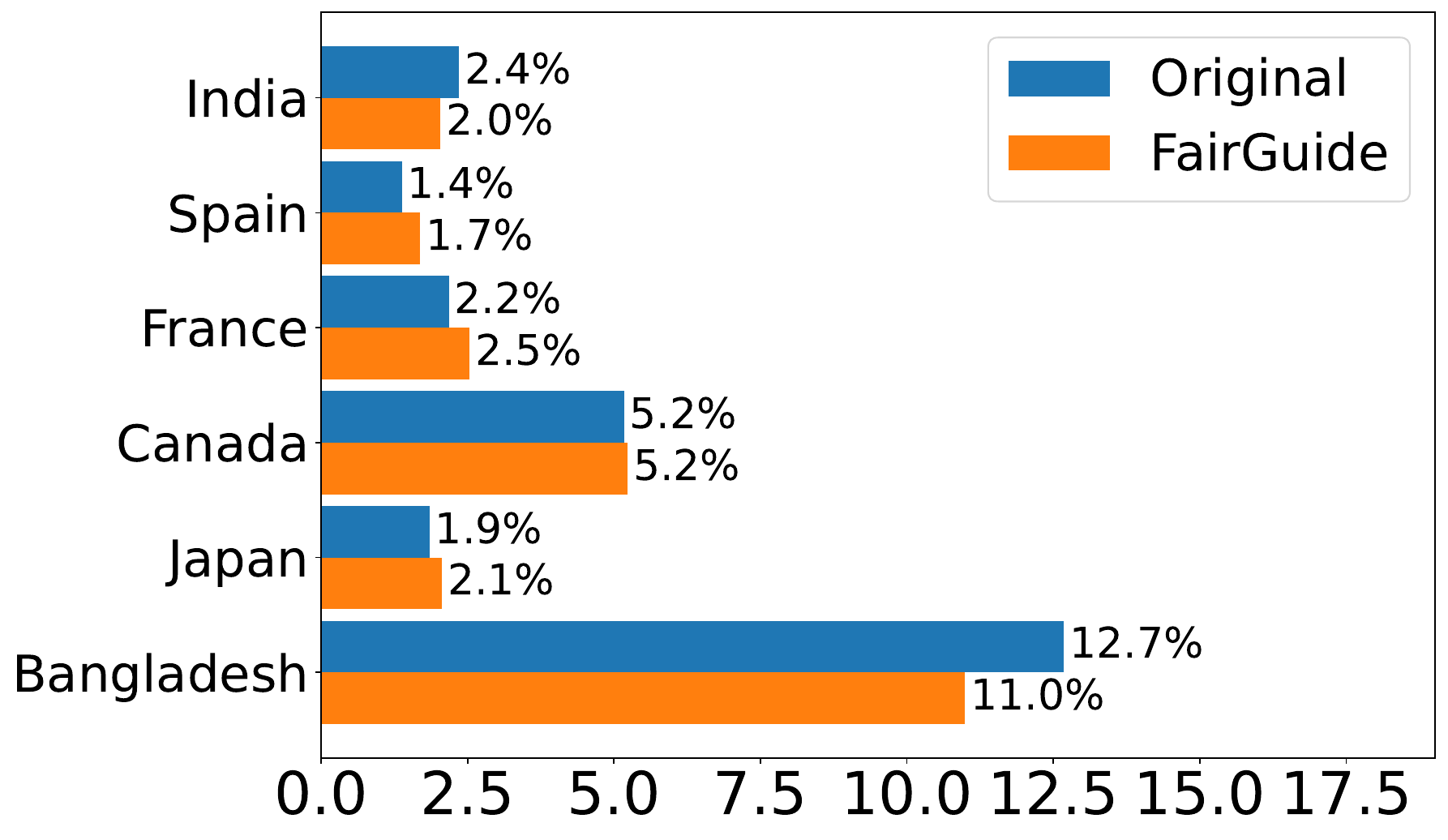}
  \vspace{-1em}
  \caption{Percentage of Links to Bangladesh Developers.}
  \vspace{-1em}
  \label{fig:analysis}
\end{figure}

\section{Algorithm of {\method}}
\label{app:algorithm}
As shown in Algorithm~\ref{alg:1}, {\method} first initializes community labels by K-means and achieve differential community detection. Then, {\method} further computes the $\Delta_{SP}$ metric and meta-gradient to the graph structure. Finally, new links sampled based on meta-gradients of structure are selected. 

\vspace{-0.5em}
\begin{algorithm}[h]
\caption{Algorithm of {\method}}
\KwIn{$A$: adjacency matrix; $X$: node feature;  $\Delta$: link addition constraint; $K$: number of  aggregation steps; $\alpha$: restart parameter; $S$: sensitive attribute vector}
\KwOut{Fairness-guided graph structure ${\mathbf{A}^\prime}$ with  new links}

${\mathbf{A}^\prime} \gets \mathbf{A}$\; 
$\mathcal{\mathbf{C_{init}}} \gets \text{K-Means}(f_\text{MLP}(X))$ \\

\While{ \text{number of added links do not reach $\Delta$}}{
    \textbf{Differential Community Detection:}\\

    Obtains the community labels $\mathbf{C^{(K)}}$ by Eq.(\ref{eq:comm_label})

    \textbf{Compute Fairness Loss:}\\
    $\Delta_{SP} \gets \Delta_{SP}(\mathbf{C^{(K)}},\mathcal{S})$ 

    \textbf{Compute Edge Score:}\\
    $\nabla_{\mathbf{A}^\prime}^{\text{meta}} \gets
    -\frac{\partial \Delta_{SP}}{\partial \mathbf{C}} \frac{\partial \mathbf{C}}{\partial {\mathbf{A}^\prime}}\odot (1 + \beta \cdot \mathbb{I}(s_i \neq s_j))$ ;

    \textbf{Gumbel Edge Sampling:}\\

    $G_{i,j} \sim \text{Gumbel}(0,1)$ \\
    $\mathbf{W} = \log( -\nabla_{\mathbf{A}^\prime}^{\text{meta}} \odot ({\mathbf{A}}^\prime == 0)$\\
    $\mathbf{P} =\frac{\log( \mathbf{W}+\epsilon)+ G_{i,j}}{\tau}$ ,
    $\delta \gets \text{Top}_k(\mathbf{P})$ 
    
    \textbf{Batch Edge Addition:}\\
    ${\mathbf{A}}^\prime \gets {\mathbf{A}}^\prime + \sum_{(i,j)\in\delta}\mathbf{E}_{i,j}$ 
}
\Return $\mathbf{A}'$;
\label{alg:1}
\end{algorithm}
\vspace{-1.5em}
\begin{table}[h]
  \centering
  \caption{ Running time of {\method} for adding one link. }
  \vskip -1em
  \label{tab:Running_time}
  \small 
  \begin{tabularx}{0.9\linewidth}{XCCC}
    \toprule
     {Dataset} &{Github} & {Pokec-n} & {Pokec-z} \\
    \midrule
      {Run Time} & 4.6 $\times 10^{-3}$s   & 1.6 $\times 10^{-2}$s  & 1.2 $\times 10^{-2}$s \\
    \bottomrule
  \end{tabularx}
  \vspace{-1em}
\end{table}

\bibliographystyle{acm}
\bibliography{references}

\end{document}